\documentclass[10pt,letterpaper,twocolumn]{article}
\usepackage[latin9]{inputenc}
\usepackage{amsmath}
\usepackage{amssymb}
\usepackage{graphicx}
\usepackage{esint}
\usepackage[unicode=true,pdfusetitle,
 bookmarks=false,
 breaklinks=true,pdfborder={0 0 1},backref=section,colorlinks=false]
 {hyperref}
\hypersetup{
 pagebackref=true,letterpaper=true,colorlinks}

\makeatletter

\pdfpageheight\paperheight
\pdfpagewidth\paperwidth

\providecommand{\tabularnewline}{\\}

\usepackage{cvpr}\usepackage{times}\usepackage{epsfig}

\usepackage{epsfig}\graphicspath{{./images/},{./images/general/},{./images/randgraph/},{./images/hepth/}}
\usepackage{bbm}\usepackage{caption}\usepackage{subcaption}
\usepackage{algorithm}\usepackage{algorithmic}

\cvprfinalcopy 

\def\cvprPaperID{20} 

\newcommand{\mb}{\mathbf}

\newtheorem{mythm}{Theorem}

\newtheorem{mydef}{Definition}

\newcommand{\denselist}{\itemsep 0pt\parsep=1pt\partopsep 0pt}
\newcommand{\bitem}{\begin{itemize}\denselist}
\newcommand{\eitem}{\end{itemize}}
\newcommand{\benum}{\begin{enumerate}\denselist}
\newcommand{\eenum}{\end{enumerate}}

\def\cl{{\cal L}}

\def\ck{{\cal K}}

\def\real{\mathbb{R}}

\def\g{\textnormal{G}}
\def\d{\textnormal{d}}
\def\st{\textnormal{s.t.}}

\ifcvprfinal\fi

\makeatother

\begin{document}

\title{Stable and Informative Spectral Signatures for Graph Matching}

\author{Nan Hu \hspace{20pt} Raif M. Rustamov \hspace{20pt} Leonidas Guibas\\
 Stanford University\\
 Stanford, CA, USA\\
 \texttt{\small nanhu@stanford.edu, rustamov@stanford.edu, guibas@cs.stanford.edu}{\small {}
}}
\maketitle
\begin{abstract}
In this paper, we consider the approximate weighted graph matching
problem and introduce stable and informative first and second order
compatibility terms suitable for inclusion into the popular integer
quadratic program formulation. Our approach relies on a rigorous analysis
of stability of spectral signatures based on the graph Laplacian.
In the case of the first order term, we derive an objective function
that measures both the stability and informativeness of a given spectral
signature. By optimizing this objective, we design new spectral node
signatures tuned to a specific graph to be matched. We also introduce
the pairwise heat kernel distance as a stable second order compatibility
term; we justify its plausibility by showing that in a certain limiting
case it converges to the classical adjacency matrix-based second order
compatibility function. We have tested our approach on a set of synthetic
graphs, the widely-used CMU house sequence, and a set of real images.
These experiments show the superior performance of our first and second
order compatibility terms as compared with the commonly used ones. 
\end{abstract}

\section{Introduction}

\label{sec:intro}

Graph matching techniques have been widely used in computer vision
in contexts such as 2D and 3D image analysis, object recognition,
biomedical identification, and object tracking. Most practical problems
require using approximate matching algorithms that can extract meaningful
correspondences even when the graphs under consideration are not isomorphic.
Therefore, when using the common quadratic assignment formulation
of graph matching problem, it is desirable to have informative first
and second order compatibility terms that are stable to violations
of isomorphism.

Spectral approaches \cite{umeyama88,cour2006,emms2009,gori2005,wilson2008,egozi2012}
have been widely used in graph matching literature. Recently, spectral
node signatures (first order compatibility terms) such as the heat
kernel signature (HKS) \cite{sog-hks-09} and the wave kernel signature
(WKS) \cite{Aubry2011} have been drawing significant attention for
matching of 3D shapes in computer graphics. These constructions are
inspired by physical processes (e.g. heat propagation) on graphs,
and are expected to inherit the physical processes' stability to perturbations
of the underlying graph. However, the analysis of stability of spectral
signatures in general is lacking, which hinders the ability to design
spectral node signatures that are not derived from physical processes.
While there has been some work on learning such signatures for 3D
shapes \cite{10.1109/TPAMI.2013.148, Aflalo2011}, they require a training set
of one form or another, which may be difficult to obtain.

The goal of our work is to establish general theoretical results about
the stability of first and second order terms constructed from graph
Laplacians \cite{hu2013arxiv}, and to use these results in a practical graph matching
framework; we are especially interested in designing node signatures
tuned to the graphs being matched. We start out with a family of spectral
node signatures, which we call the Laplacian Family Signatures (LFS).
This family is parametrized by a real-valued function of two variables
-- the construction filter; particular choices of the construction
filter yield the HKS and WKS as special cases. We first establish
a stability result for the LFS, obtaining an upper bound on how much
the signature of a node can change under perturbations of the underlying
graph. Next, we relax the bounds in our stability theorem, which allows
us to encode both the requirements of stability and informativeness
in a single objective function. By optimizing this objective, we obtain
a custom construction filter and so a custom spectral node signature
for matching the graph under consideration.

While the steps above yield a stable first order compatibility term,
stability of the second order compatibility term is as important.
Another contribution of this work is to introduce the pairwise heat
kernel distance as a second order compatibility term suitable for
inclusion into the Integer Quadratic Program (IQP) formulation of
the graph matching problem. This term can be shown to be stable to
perturbations of the underlying graph. To justify its use as a second
order term we prove that in a certain limiting case, the pairwise
heat kernel distance reproduces the commonly used term obtained from
the adjacency matrices.

Our overall practical graph matching approach has a number of benefits.
First, our approach is based on rigorous theoretical results on stability
of node and edge signatures, that are of independent interest. Second,
despite the fact that our construction filter is non-parametric, we
obtain a simple convex optimization problem that can be solved efficiently.
Third, in contrast to previous methods, e.g. \cite{Aflalo2011}, we do not require a training
set, but base our node signature optimization on the given graph to
be matched. Ideally, signature design should be based on a representative
``average'' graph of the collection of graphs arising in a specific
context. However, computing the average graph itself requires reliable
matchings between all graphs in the collection, leading to a chicken-and-egg
problem. To circumvent this, we hypothesize that an attempt to match
a given graph to other graphs in a collection -- if it is to be at
all successful -- is indicative of shared structure, and this should
allow optimizing signatures based solely on the graph that is being
matched. Our experimental results confirm that signatures optimized
in such a way provide superior results over un-optimized signatures
in a number of settings.

The rest of the paper is organized as follows. After reviewing previous
work in Section \ref{ssect:relwork}, we introduce and analyze the
first order compatibility terms in Section \ref{sect:specdscp}. The
second order compatibility term is introduced in Section \ref{sec:Noise-Tolerant-Second}.
The IQP for quadratic assignment formulation of graph matching problem
is set up in Section \ref{sect:mscheme}. We present our experimental
results on three different graph matching tasks in Section \ref{sect:exp}.

\section{Related Work}

\label{ssect:relwork}

Node-based signatures have been popular in the context of graph matching.
Joilli et. al.\cite{jouili2009} proposed a signature composed of
the degree of the a node followed by the ordered weights of each incident
edge and padded with zero if necessary. 
Gori et. al. \cite{gori2005} constructed node signatures from the
steady state distributions of simulated random walks similar to PageRank.
Eshera \cite{eshera1984} built signatures for attributed relational
graphs (ARG). 
Shokoufandeh et. al \cite{shok1999} constructed feature-based node
signatures for bipartite matching. 
The same authors \cite{shok2001} later proposed a topological signature
vector (TSV) for directed acyclic graphs (DAG). Hu et. al \cite{hu2013}
considered graph matching from the learning of a signature-based proximity
matrix using disclosed known correspondences without explicitly computing
the signatures themselves.

Node signatures used in our work are most closely related to spectral
methods for graph matching. Among the pioneering works is Umeyama's
\cite{umeyama88} weighted graph matching algorithm 
, which was later generalized to graphs of different sizes \cite{luo2001,zhao2007}.
Robles-Kelly et. al. \cite{robles-kelly2002} ordered the nodes from
the steady state of Markov chain 
with the edge connectivity constraint and matched using edit distance;
in \cite{robles-kelly2005}, they ordered the nodes using the leading
eigenvector of the adjacency matrix. 
Qiu et. al \cite{qiu2006} considered using Fiedler vector 
to partition the graph for hierarchical matching; 
their method, however, works only on planar graphs. 
Cho et. al. \cite{cho2010} constructed reweighted random walks similar
to personalized PageRank on the association graph with the addition
of an absorbing node. They computed its quasi-stationary distribution
and discretized the continuous solution to find a matching. 
Emms et. al. \cite{emms2009} built an auxiliary graph from the two
graphs 
and simulated a quantum walk. 
Particle probability of each auxiliary node was used as the cost of
assignment for a bipartite matching. 

In a broader sense of relatedness to our work are other relaxation-based
matching algorithms. 
Gold and Rangarajan \cite{gold1996} proposed the well-know \textit{Graduated
Assignment Algorithm}. 
van Wyk et. al. \cite{vanwyk2004} designed a projection onto convex
set (POCS) based algorithm to solve IQP by successively projecting
the relaxed solution onto the convex constraint set. Schellewald et.
al. \cite{schell2005} constructed a semidefinite programming relaxation
of the IQP. Leordeanu et. al. \cite{leord2005} proposed a spectral
method to solve a relaxed IQP where they drop the linear inequality
constraint during relaxation and only incorporate it at the discretization
step. The idea was further extended by Cour et. al. \cite{cour2006},
where they added an affine constraint during relaxation. Zaslavskiy
et. al. \cite{zas2009} approached the IQP from the point of a relaxation
of the original least-square problem to a convex and concave optimization
problem on the set of doubly stochastic matrices. Leordeanu et. al.
\cite{leord2009} proposed an integer projected fixed point (IPFP)
algorithm to solve the quadratic assignment problem. Zhou et. al.
\cite{zhou2012} proposed a factorized graph matching algorithm to
solve IQP problem by factorizing the affinity matrix into the Kronecker
product of smaller matrices.

\section{First Order Compatibility}

\label{sect:specdscp} In this section, we introduce the Laplacian
Family Signatures (LFS) as a structural descriptor for graph nodes.
We then establish a stability theorem showing the robustness of these
descriptors to perturbations of the graph. Next, for a given graph,
we will show how the bounds established by our stability theorem together
with considerations of informativeness can be used to choose an optimal
signature from this family of signatures.

\subsection{Laplacian Family Signatures}

Consider one of the graphs to be matched, say $G=(V,E)$. Let $w$
be the weights on edges, i.e. $w:E\mapsto\mathbb{R}^{+}$. The graph
Laplacian is defined as $\cl=D-A$, where $A$ is the graph adjacency
matrix, and $D$ is a diagonal matrix of total incident weights, i.e.
$D_{ii}=\sum_{j}{A_{ij}}$. $\cl$ has numerous useful properties
\cite{biyi2007}, of which most relevant to us is its symmetry and
positive semi-definiteness. This makes it possible to consider the
eigen-decomposition of $\cl$; we denote by $\{\lambda_{k},\phi_{k}\}_{k=1}^{|V|}$
the eigenpairs (eigenvalue and associated eigenvector) of the graph
Laplacian matrix $\cl$.

The eigenvalues and eigenvectors of the Laplacian matrix carry a wealth
of structural information about the underlying weighted graph. Our
goal is to use this information to obtain signatures for nodes of
the graph that are both stable and informative. We first start with
a very general definition of a family of signatures.

\label{ssect:lfs}\begin{mydef} \label{def:lfs} For a given real
valued function $h(\cdot;\cdot):\mathbb{R}_{+}^{2}\rightarrow\mathbb{R}$,
the Laplacian Family Signature (LFS) of a node $i\in V$ is a one-parameter
family of structural node descriptors that is defined by 
\begin{equation}
s_{i}(t)=\sum_{k}h(t;\lambda_{k})\phi_{k}(i)^{2}.\label{eqn:lfs}
\end{equation}
We refer to $h(\cdot;\cdot)$ as the construction filter.

\end{mydef}

The Laplacian Family Signatures describe a given node's structural
relationship to its neighborhood at large (see e.g. physical interpretation
of signatures in the next subsection). Note that the signature of
a given node $i\in V$ is itself a function $s_{i}(\cdot):\mathbb{R}_{+}\rightarrow\mathbb{R}$.
Thus, two nodes $i$ and $a$ from the same or different graphs can
be compared by using any kind of distance/norm between the functions
$s_{i}(\cdot)$ and $s_{a}(\cdot)$.

A number of particular choices of the construction filter have been
considered in previous work. Choosing $h(t;\lambda_{k})=\exp(-t\lambda_{k})$
results in the heat kernel signature (HKS) \cite{sog-hks-09}, and
selecting $h(t;\lambda_{k})=\exp(-\frac{(t-\log\lambda_{k})^{2}}{2\sigma^{2}})$,
one obtains the wave kernel signature (WKS) \cite{Aubry2011}; these
signatures were shown to have desirable properties for applications
in 3D shape analysis and matching. Assuming that $h(t;\lambda_{k})=g(t\lambda_{k})$
is a band-pass filter with a special behavior as in \cite{HAMMOND2011},
we can easily obtain another signature -- the wavelet kernel signature.

It is clear that a plethora of such descriptors can be obtained by
simply varying the construction filter. An important question is then
what choice of the filter is optimal in one or another sense. We address
this issue in the following two subsections by considering two conflicting
requirements -- the informativeness and stability of signatures.

\subsection{\label{sub:General-Stability-Result}General Stability Result}

The LFS signatures are naturally intrinsic: if two graphs are isomorphic,
then the signatures of corresponding nodes are the same. However,
for a signature to be practically useful, it should also be stable
under perturbations of the graph. Stability of existing signatures,
such as HKS and WKS, are derived from intuitive considerations based
on physical interpretations. For example, HKS has an interpretation
in terms of a simulated heat diffusion process \cite{sog-hks-09}:
for each node, this signature captures the amount of heat left at
the node at various times (here $t$) assuming that a unit amount
is put on the node initially ($t=0$). WKS also has a physical interpretation
in terms of a quantum mechanical process on the graph \cite{Aubry2011}.
The stability then follows from the assumption that these physical
processes are stable under small perturbations of the underlying graph.

One of our main results is to establish the stability of LFS signatures
in general without a recourse to a physical interpretation. Importantly,
we obtain upper bounds on how much the signatures may change, and
we consider both the case of distinct and repeating eigenvalues.

\begin{mythm}\label{thm-stability-theorem}Let $A,A'$ be the adjacency
matrices of a pair of graphs, and $\cl,\cl'$ be the induced Laplacians.
Let the size of the graph be $n$, and $\lambda_{1}<\cdots<\lambda_{k}$
denote the $k$ distinct eigenvalues of $\cl$. Let $s_{i}(t)$ and
$s'_{i}(t)$ denote the LFS's of node $i$. Assume $\lambda_{j+1}-\lambda_{j}\geq\delta,\forall j$,
$\|A-A'\|_{F}\leq\frac{\epsilon}{\sqrt{n+1}}<\delta$, and $h(\cdot;\cdot)\in C^{2}(\real_{+}^{2})$.
If $k=n$ (non-repeating eigenvalues), we have 
\[
|s'_{i}(t)-s_{i}(t)|\leq C_{0}(\delta,t)\epsilon,
\]
where $C_{0}(\delta,t)$ is a constant independent of $\epsilon$.
If $k<n$ (repeating eigenvalues), we have 
\[
|s_{i}(t)-s'_{i}(t)|\leq C_{1}(t)\left(\frac{\delta}{\delta-\epsilon}-1\right)+C_{2}(t)\epsilon,
\]
where $C_{1}(t),C_{2}(t)$ are constants depending only on $t$. \end{mythm}

See supplementary material for a proof. Even a more general stability
result when the number of graph nodes changes can be obtained using
lifting ideas similar to \cite{shok2001}.

\subsection{Signature Optimization}

In addition to stability, a practically useful signature has to be
informative. For a special category of compact Riemannian manifolds,
both HKS and WKS have been shown to be informative, e.g. Theorem 1
in \cite{sog-hks-09}, i.e. they fully characterize the shapes up
to isometry. For general graphs, however, the informativeness could
hardly be well-defined, as one could easily find counter examples
such that non-isomorphic graphs share the same LFS. In the context
of graph matching, nonetheless, we could still intuitively explain
informativeness of LFS as structural richness. This is why instead
of a single number, the LFS describes a node $i\in V$ by a function
$s_{i}(\cdot):\mathbb{R}_{+}\rightarrow\mathbb{R}$. However, this
does not directly guarantee the informativeness of the signature,
because the function values $s_{i}(t)$ may be strongly correlated
with each other at different values of $t$, reducing the information
content of the descriptor.

Interestingly, the requirements of informativeness and stability are
conflicting. Indeed, our stability theorem shows that the signatures
are more stable when the construction filter is smooth. On the other
hand, the information content is maximized when the construction filter
concentrates as much as possible (like a delta function) at a given
eigenvalue, thereby extracting information from non-overlapping frequency
bands of eigenvectors.

Both of these conflicting requirements can be captured by relaxing
the bounds established in Theorem \ref{thm-stability-theorem}. As
can be seen (c.f. supplementary material) in both distinct ($k=n$)
and repeated ($k<n$) eigenvalue cases, the upper bound depends on
two terms -- $\max_{j}|h(t,\lambda_{j})|$ and $\max_{j}|\frac{\partial}{\partial\lambda}h(t,\lambda_{j})|$.
Taking the $L^{1}$-norm as the distance metric when comparing signatures,
we can bound the change in signature of a node $i\in V$ under perturbation
(c.f. Theorem \ref{thm-stability-theorem}) as 
\begin{eqnarray}
d(s_{i},s'_{i}) & = & \int|s_{i}(t)-s'_{i}(t)|\d t \nonumber \\
 & \leq & \frac{2\epsilon}{\delta}\left(\int\max_{j}\left|h(t;\lambda_{j})\right|\d t\right)\nonumber\\
 &  & +\epsilon\left(\int\max_{j}\left|\frac{\partial}{\partial\lambda}h(t;\lambda_{j})\right|\d t\right)
\label{eqn:lfs-ubound}
\end{eqnarray}
To minimize this upper bound, we, nevertheless, have two contradicting
terms unless we use the trivial solution $h\equiv0$ everywhere. The
first term requires $h(t;\lambda)$ concentrate on $\lambda$, i.e.
it should be as narrow as possible on each distinct frequency $\lambda_{j}$
(informativeness). The second term, on the other hand, requires $h(t;\lambda)$
to be smooth at $\lambda$, i.e. it should be as wide as possible
at each distinct frequency $\lambda_{j}$ (stability).

When a meaningful matching between two graphs exists, it is natural
to assume that one of the graphs is a perturbation of the other. Based
on this intuition, we propose to pick one of the graphs as the source
graph, and to determine an optimal (with respect to the upper bound
above) construction filter $h(t;\lambda)$ for the source graph. Then
the graphs are matched using this optimal filter on both the source
and target graphs.

To find the optimal construction filter $h(\cdot;\cdot)$ for a given
source graph, we will minimize the upper bound above. To simplify
our optimization problem, note that for informativeness, we need $h(t;\lambda)$
to be large at $\lambda$ while fading away farther from $\lambda$,
and so we assume $h(t;\lambda)$ to be of the form $h(|t-\lambda|)$
with $h(0)=1$. We will also restrict $h$ to be positive and uniformly
decreasing. Finally, the requirement of $h(\cdot)\in C^{2}(\real_{+}^{2})$
is achieved by putting a bound on the second derivative $h''(\cdot)$.

The simplified optimization problem becomes that of finding $h:\mathbb{R}_{+}\rightarrow\mathbb{R}$
solving

\[
\begin{array}{ll}
\min & \mu\int\max_{j} h(|t-\lambda_{j}|)\d t+\int\max_{j}|h'(|t-\lambda_{j}|)|\d t\\
\st & h(0)=1,\: h(x)\geq0,\: h'(x)\leq0,\:|h''(x)|\leq c_{h}
\end{array}
\]
where the parameter $\mu=\frac{2}{\delta}$ is expressed in terms
of the eigen-gap $\delta$. In practice, we set $\delta$ to be the
average eigen-gap of the eigenvalue sequence. As it is clear from
the objective function, in addition to the eigen-gap, the optimal
filter will depend on the entire eigenvalue distribution of the Laplacian.
%
%

While the above problems may be hard to solve analytically, it can be straightforwardly
discretized and solved numerically as a convex optimization problem.
In practice, w.l.o.g., we consider $h(|t|)$ to be non-zero only on $[-T,T]$. Symmetric as it is, we only need to evenly sample $h$ on $[0, T]$ as a vector $\mb{h} = \left[h_0,\cdots,h_N\right]^\top$, with $h_0=1,h_N=0$. Let $\Delta h$ be the step between samples. The first and second order derivative could be numerically estimated as $\mb{h}' = \left[\frac{h_1-h_0}{\Delta h}, \cdots, \frac{h_N-h_{N-1}}{\Delta h}\right]^\top$ and $\mb{h}'' = \left[\frac{h_2-h_0}{\Delta h^2}, \cdots, \frac{h_N-h_{N-2}}{\Delta h^2}\right]^\top$. Let $\mb{h}_j$, $\mb{h}'_j$, $\mb{h}''_j$ be the $\lambda_j$-shifted version of $\mb{h}$, $\mb{h}'$, $\mb{h}''$ (padded with zeros if needed). The discretized optimization problem could therefore be written as
\[
\begin{aligned} 
\min~ & ~\mu\mb{1}^\top\max(\mb{h}_1,\cdots,\mb{h}_k) + \mb{1}^\top\max(\mb{h}'_1,\cdots,\mb{h}'_k)\\
\text{s.t.}~ & ~\mb{h}_i\geq0,~\mb{h}'_i\leq0,~-c_h\leq\mb{h}''_i\leq0,~h_0=1,~h_N=0
\end{aligned}
\]
where $\mb{1}$ is the unit vector and $\max(\cdot)$ is the element-wise max among its arguments.

Figure \ref{fig:opt-kernel}
shows the optimal filters for several types of randomly generated
graphs %
\footnote{The random graphs tested in Figure \ref{fig:opt-kernel} are as follows:
1) complete graph with uniformly distributed edge weights; 2) complete
graph with exponential distributed edge weights; 3) complete graph
with half normal distributed edge weights; 4) preferential attachment
model \cite{newman10networks} with uniformly distributed edge weights,
with the number of new connections being 2. The mean edge weight for
all random graphs is the same. %
}. The LFS with our optimized kernel is named ''adaLFS'' (for adaptive
LFS) in all the figures that follow. The improved matching results
on some well-known dataset can be seen in Figures \ref{fig:sigperf},
\ref{fig:houseperf}, and will be more thoroughly discussed in Section
\ref{sect:exp}.

\section{\label{sec:Noise-Tolerant-Second}Stable Second Order Compatibility}

In this section, we introduce a second order compatibility term based
on the heat diffusion process on graphs. Specifically, consider the
graph heat kernel $k_{t}(i,j)$, which measures the amount of heat
transferred from node $i$ to node $j$ after time $t$, assuming
a unit amount was placed at $i$ in the beginning ($t=0$). The heat
kernel has the following representation in terms of the eigen-decomposition
of the graph Laplacian: 
\[
k_{t}(i,j)=\sum_{k}\exp(-t\lambda_{k})\phi_{k}(i)\phi_{k}(j).
\]
Using an argument similar to the proof of Theorem \ref{thm-stability-theorem},
one can establish stability of $k_{t}(\cdot,\cdot)$ to perturbations
(see supplementary material). Therefore, it provides a natural choice
for a second order compatibility term.

\begin{mydef} Let $k_{t}$ and $k'_{t}$ be the heat kernels of $G=(V,E)$
and $G'=(V',E')$ respectively. For $i,j\in V$ and $a,b\in V'$,
the pairwise heat kernel distance is defined as 
\[
d_{t}^{\ck}(i,j,a,b)=\left|k_{t}(i,j)-k'_{t}(a,b)\right|.
\]
\end{mydef}

We will compare this term with the commonly used term based on the
adjacency matrices $A,A'$ of the graphs $G,G'$: 
\[
d^{A}(i,j,a,b)=\left|A_{ij}-A'_{ab}\right|.
\]
The following theorem shows that the pairwise heat kernel distance
$d_{t}^{\ck}$ is a stable approximation of the pairwise adjacency
distance $d^{A}$ (see supplementary material for a proof).

\begin{mythm} Let $d_{t}^{\ck}(i,j,a,b)$ and $d^{A}(i,j,a,b)$ be
the pairwise heat kernel distance and pairwise adjacency distance
for graph $G$ and $G'$, then the following holds: 
\[
\lim_{t\rightarrow0}\frac{d_{t}^{\ck}(i,j,a,b)}{t}=d^{A}(i,j,a,b).
\]
\end{mythm}

When $t$ is small, $d_{t}^{\ck}$ is a good approximation of $d^{A}$;
as $t$ increases, $d_{t}^{\ck}$ is smoothed out. In this way it
becomes stable, because in the ideal case when the graphs are isomorphic,
$d^{A}$ should be zero everywhere for matched pairs.

\begin{figure*}[ht]
\begin{subfigure}{.32\linewidth} \includegraphics[width=0.96\linewidth]{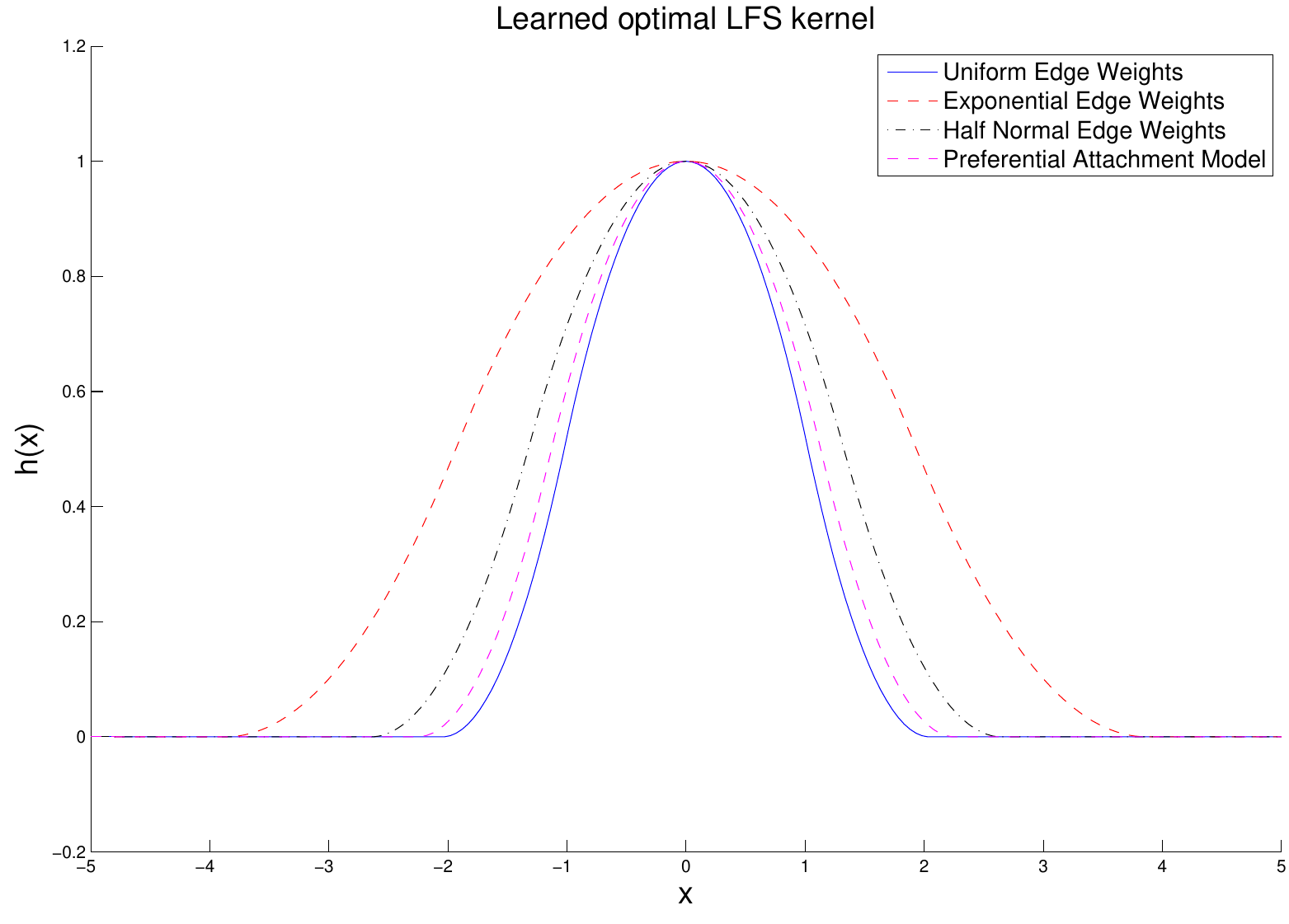}

\caption{}

\label{fig:opt-kernel} \end{subfigure} \begin{subfigure}{.32\linewidth}
\centering \includegraphics[width=0.96\linewidth]{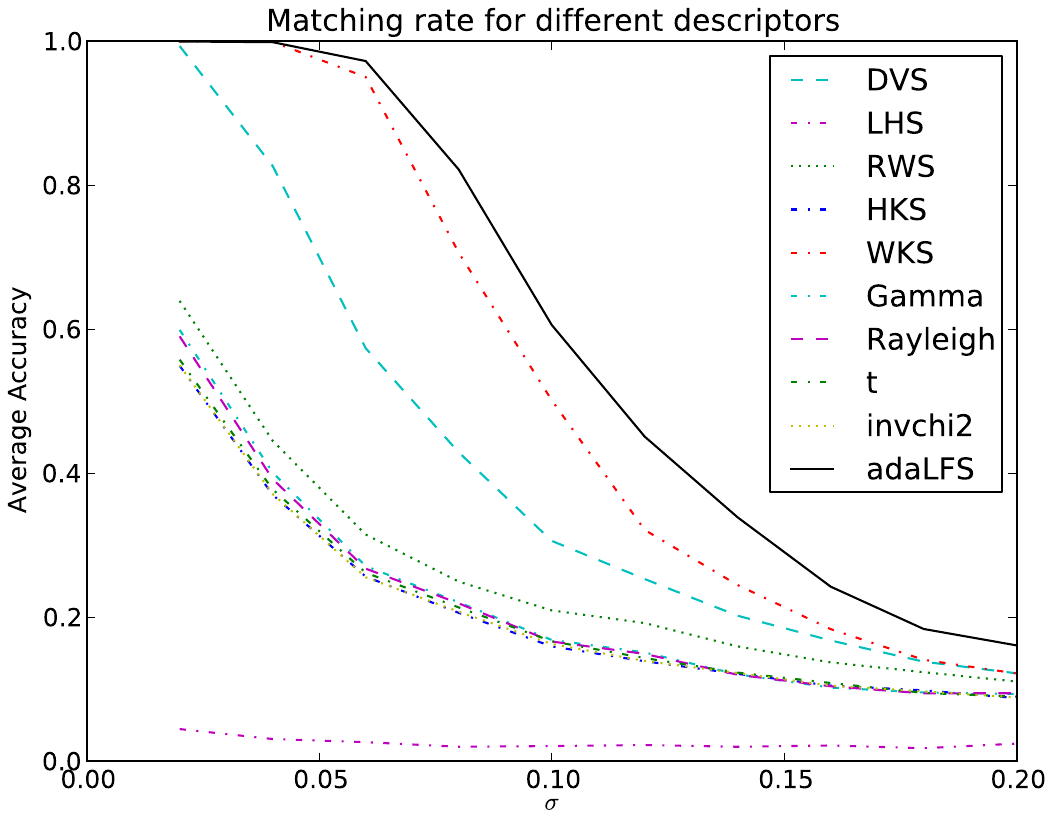} \caption{}

\label{fig:sigperf} \end{subfigure} \begin{subfigure}{.32\linewidth}
\centering \includegraphics[width=0.96\linewidth]{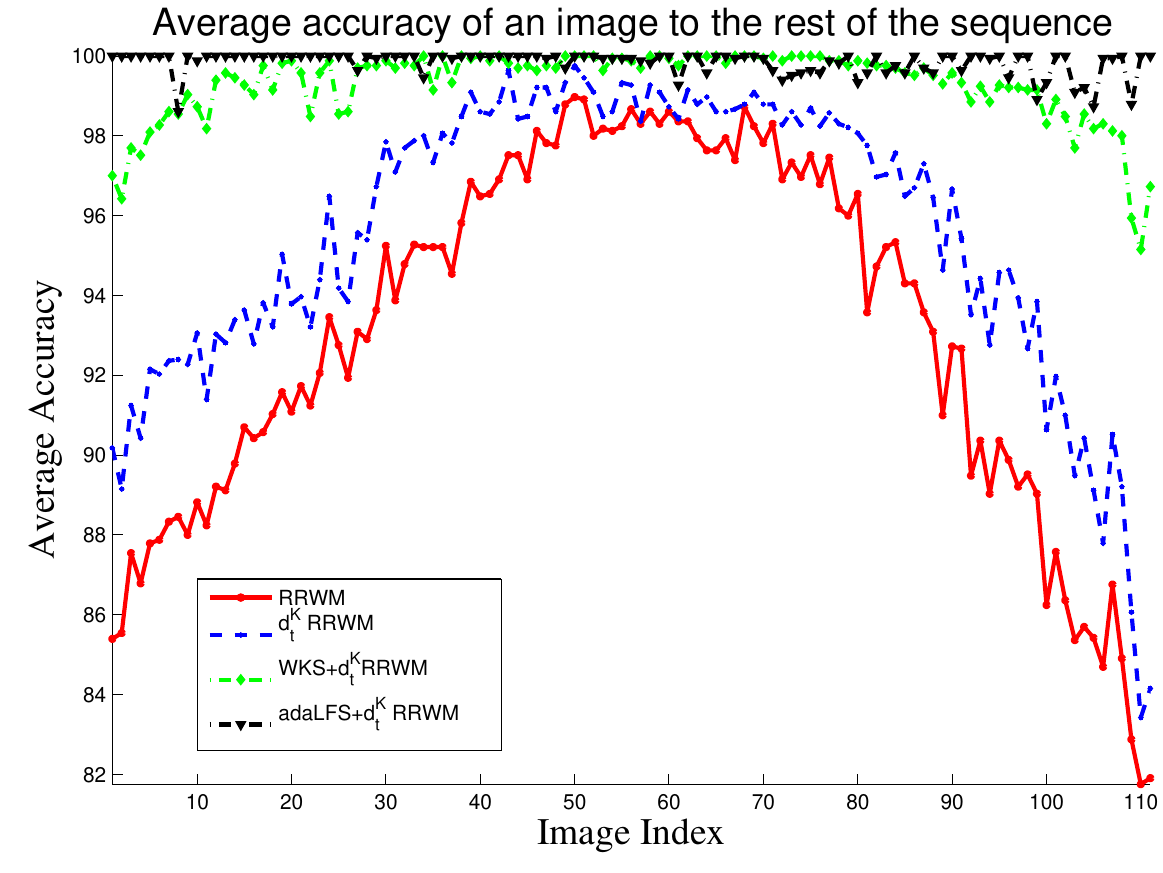} \caption{}

\label{fig:houseperf}

\end{subfigure} \caption{(a) Optimal construction filters for different types of random graphs
(red) and an image key points graph (blue). (b) Signature comparison
for $G(n,m)$ random graphs. (c) Matching accuracy for CMU House sequence.}
\end{figure*}

\section{Matching Scheme}

\label{sect:mscheme}

To directly compare the performance within LFS and with other node
signatures, the problem is cast as a bipartite graph matching problem
as in existing node signature based matching work \cite{shok1999,gori2005,jouili2009},
where costs are set as the distances between signatures. The problem
is solved using Hungarian algorithm \cite{kuhn1955}.

For practical matching, in addition to node signatures, we use the
pairwise heat kernel distance $d_{t}^{\ck}$ as the second order constraint
and formulate the problem as an integer quadratic program (IQP). Namely,
for two graphs $G=(V,E)$ and $G'=(V',E')$ to be matched and nodes
$i\in V,a\in V'$, let $d(i,a)$ be the distance between their node
signatures. We construct the compatibility matrix $W\in\real^{|V||V'|\times|V||V'|}$
as 
\[
W_{ia,jb}=\begin{cases}
d_{t}^{\ck}(i,j,a,b) & i\neq j,~a\neq b\\
\alpha d(i,a) & i=j,~a=b
\end{cases}
\]

Letting $\mb{X}\in\{0,1\}^{|V|\times|V'|}$ be the one-to-one mapping
matrix, and $\mb{x}\in\{0,1\}^{|V||V'|}$ its vectorization, the IQP
can be written as 
\[
\begin{aligned} & ~\mb{x}^{*}=\arg\min(\mb{x}^{\top}W\mb{x})\\
\text{s.t.}~ & ~\mb{x}\in\{0,1\}^{|V||V'|},~\forall i~\sum_{a\in V'}\mb{x}_{ia}\leq1,~\forall a~\sum_{i\in V}\mb{x}_{ia}\leq1
\end{aligned}
\]

As is well-known, this problem is NP-complete and there is a large
literature of approximation algorithms. In our experiments, we selected
a recently proposed algorithm, the reweighed random walk matching
(RRWM) \cite{cho2010} because of its superior performance compared
with other state-of-the-art approximation algorithms, including SM
\cite{leord2005}, SMAC \cite{cour2006}, HGM \cite{zass2008}, IPFP
\cite{leord2009}, GAGM \cite{gold1996}, SPGM\cite{vanwyk2004}. While finding a good solver for IQP is an interesting problem {\it per se}, we will not explore possibilities in this direction as it will inevitably shift the focus of our paper.

\section{Experiments}

\label{sect:exp}

We tested our descriptor on three different datasets: 1) synthetically
generated random graphs; 2) CMU House sequence for point matching;
3) feature matching using real images.

\subsection{Synthetic Random Graphs}

\label{ssect:randgraph} In this section, following the experimental
protocol of \cite{cho2010}, we synthetically generate random graphs
and perform a comparative study. In the first part of the experiment,
we use random graphs from Erdös-Rényi model $\g(n,m)$, where $m$
edges are randomly selected from all possible $\frac{n(n-1)}{2}$
edges. For each selected edge, we add a uniform random weight in the
range $[0,1]$. The graph is then perturbed by adding random Gaussian
noise ${\cal N}(0,\sigma^{2})$ on selected edges.

In this test, we compare the performance within our LFS and with existing
node signatures, namely degree vector signature (DVS) \cite{jouili2009},
local histogram signature (LHS) \cite{Wong2006}, random walk signature
(RWS) \cite{gori2005}. For comparisons within LFS family, we do not
limit ourselves to HKS and WKS, which are known special cases of LFS.
Noticing that both HKS and WKS construction filters are continuous
probability density functions (pdf) of some distribution, we include
a number of other pdfs: i) Gamma distribution $(t\lambda_{i})^{k-1}\exp\left(-t\lambda_{i}/\theta\right)$,
ii) Gaussian distribution $\exp\left(-(t-\mu(\lambda_{i}))^{2}/2\sigma^{2}\right)$,
iii) t-distribution $\left(1+(t\lambda_{i})^{2}/\theta\right)^{-(k+1)/2}$,
iv) Rayleigh distribution $(t\lambda_{i})^{k-1}\exp\left(-(t\lambda_{i})^{2}/2\theta^{2}\right)$,
and v) Inverse Chi-square distribution $(t\lambda_{i})^{-k/2-1}\exp\left(-\theta/2t\lambda_{i}\right)$.
This set of construction filter choices are by no means exhaustive;
we hope, however, it will illustrate the improved performance of our
adaptive LFS (adaLFS).

First, to directly compare the performance of node signatures, we
use bipartite matching as the matching scheme. In the experiment,
we set $n=50$ and $m$ is uniformly in $[400,1000]$, and generate
$100$ pairs of graphs. Fig. \ref{fig:sigperf} shows the average
accuracy (i.e. the fraction of correct matches over ground truth matches)
over the amount of noise added to the graph. As can be seen, our adaLFS
exhibits best performance of all node signatures considered. Since
WKS is the second best signature, we will include it in all of the
comparisons that follow.

\begin{figure*}[ht]
\centering \begin{subfigure}{.32\linewidth} \includegraphics[width=0.96\linewidth]{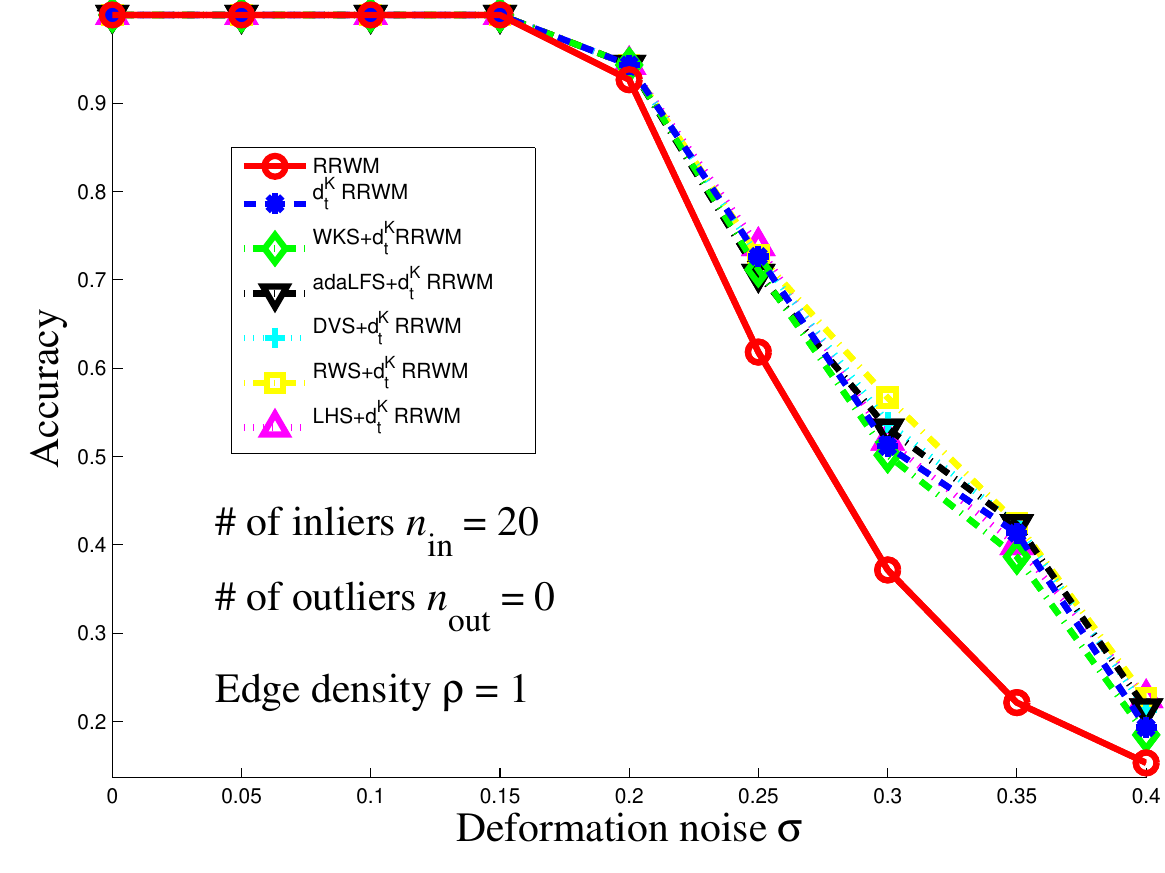}
\caption{\label{fig:def-noise}}

\end{subfigure} \begin{subfigure}{.32\linewidth} \includegraphics[width=0.96\linewidth]{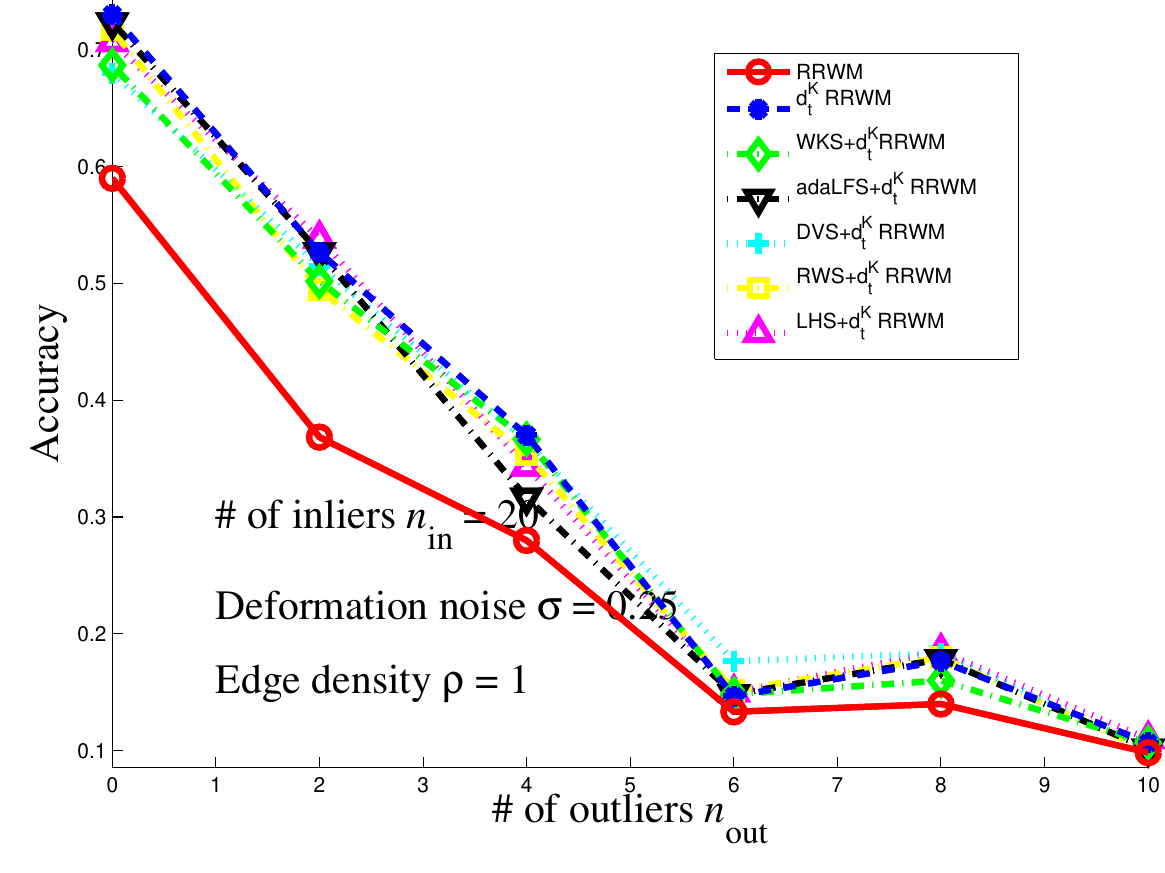}
\caption{}

\end{subfigure} \begin{subfigure}{.32\linewidth} \includegraphics[width=0.96\linewidth]{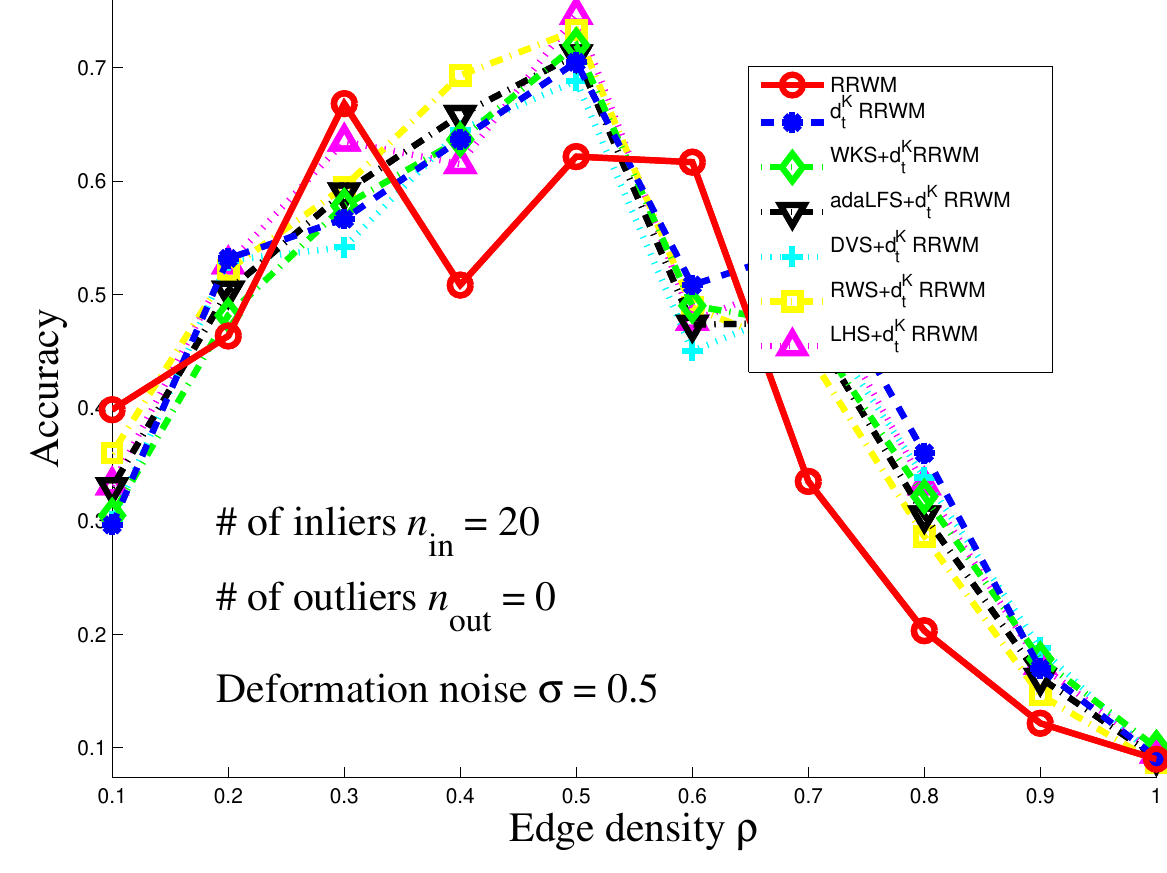}
\caption{}

\end{subfigure} \caption{Matching accuracy of in the IQP setting.}

\label{fig:randiqp} 
\end{figure*}

Second, we test different node signatures together with $d_{t}^{\ck}$
as the pairwise constraint in the IQP setting. The random graphs are
generated according to \cite{cho2010} using their publicly available
code. For a pair of graph $G_{1}$ and $G_{2}$, they share $n_{\text{in}}$
common nodes and $n_{\text{out}}^{(1)}$ and $n_{\text{out}}^{(2)}$
outlier nodes. Edge weights are randomly distributed in $[0,1]$,
and random Gaussian noise ${\cal N}(0,\sigma^{2})$ is added.

In this experiment, we test the matching performance where the IQP
compatibility matrix $W$ includes: i) only $d^{A}(i,j,a,b)$, ii)
only $d_{t}^{\ck}(i,j,a,b)$, iii) $d_{t}^{\ck}(i,j,a,b)$ with different
node signatures, on three different settings: 1. different level of
deformation noise $\sigma$; 2. different number of outliers; 3. different
edge densities $\rho$. Fig. \ref{fig:randiqp} shows the average
matching accuracy. The baseline, shown in red solid curve, is RRWM
using pairwise adjacency distances $d^{A}$ only. With $d_{t}^{\ck}$
substituting $d^{A}$, the matching performance is more tolerant to
noise. Comparing Fig. \ref{fig:sigperf} and Fig. \ref{fig:def-noise},
it can be seen that the large performance gap among node signatures,
however, was marginalized out because of the second order compatibility
constraint $d_{t}^{\ck}$. As shown in Fig. \ref{fig:randiqp}, the
matching accuracy of our proposed method is superior to the baseline
algorithm (red solid line) in all three noisy settings.

\subsection{CMU House Sequence}

\label{ssect:houseexp} In this experiment, we use the CMU House sequence
to test our descriptors. This sequence has been widely used to test
different graph matching algorithms. It consists of 110 frames, and
there are 30 feature points labeled consistently across all frames.
We build fully connected graphs purely based on the geometry of the
feature points, taking the exponential of the normalized Euclidean
distance of the key points as the weights between pair of feature
points. Note this graph setup is different from \cite{cho2010}. In
their original work, they use the Euclidean distance as edge weights,
which could be seen as a dissimilarity measure. While in our frame
work, we used the normalized exponential of the dissimilarity as a
similarity measure, which conforms with the physical meaning of neighboringness.
IQP compatibility matrices $W$ are set up similarly as in Section
\ref{ssect:randgraph}. In the first part of the experiments, We compute
the average matching accuracy of each frame to the rest of frames
in the sequence. Fig. \ref{fig:houseperf} shows the accuracy of the
matching. As can be seen the matching performance improves when $d_{t}^{\ck}$
is used to substitute $d^{A}$. With adaLFS as the first order compatibility,
furthermore, the matching accuracy is improved even further. Fig.
\ref{fig:houseplot} shows an example of the matching between the
first and the last frames of the sequence.

\begin{figure}[ht]
\centering \begin{subfigure}{.32\linewidth} \centering \includegraphics[width=0.96\linewidth]{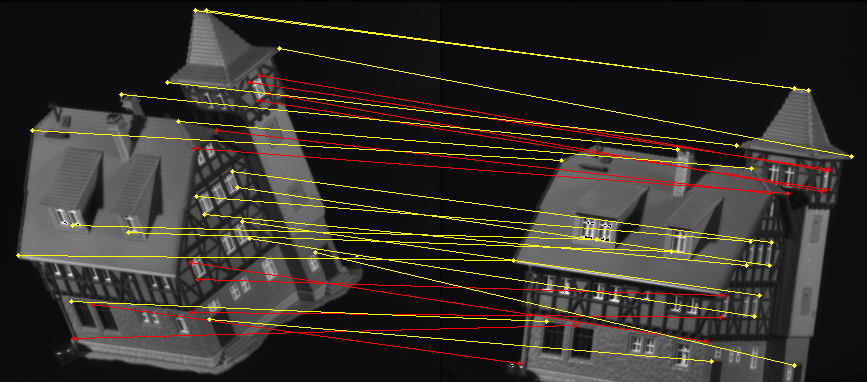}
\caption{{\scriptsize RRWM}}

\end{subfigure} 
\begin{subfigure}{.32\linewidth} \centering \includegraphics[width=0.96\linewidth]{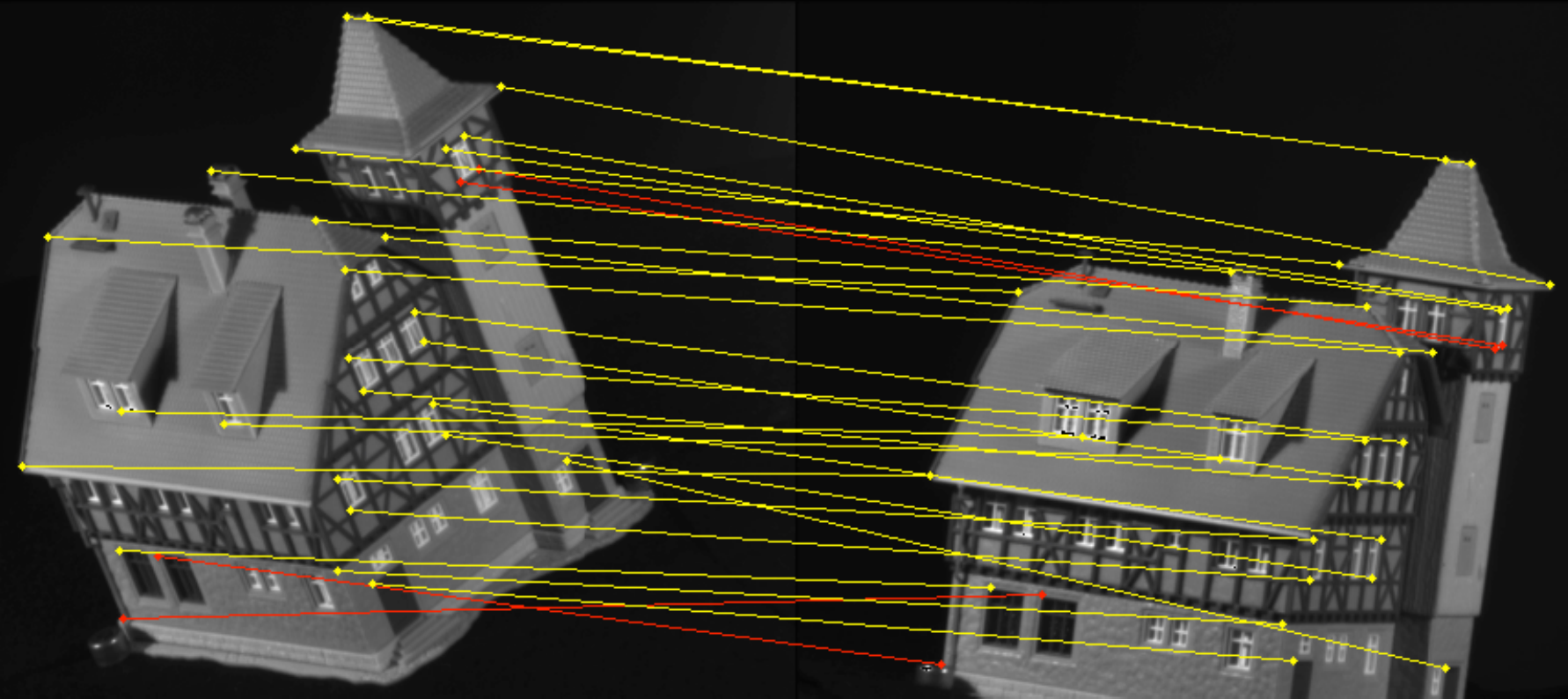}
\caption{{\scriptsize WKS+$d_{t}^{\ck}$ RRWM}}

\end{subfigure} \begin{subfigure}{.32\linewidth} \centering
\includegraphics[width=0.96\linewidth]{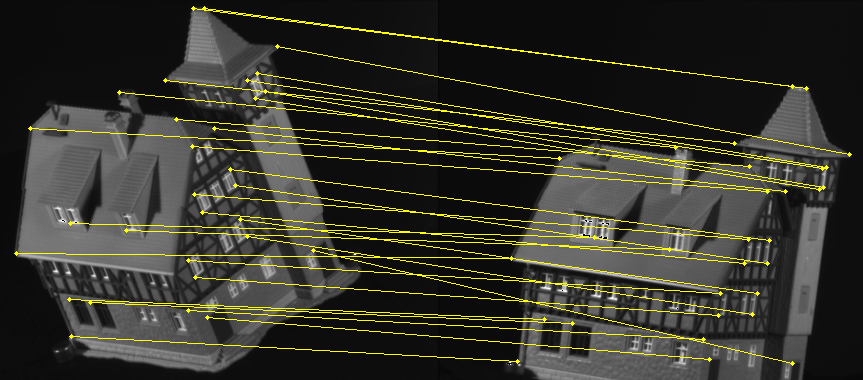} \caption{{\scriptsize adaLFS+$d_{t}^{\ck}$ RRWM}}

\end{subfigure} \caption{Example matching on House sequence. Yellow lines are correct matches,
and red lines are incorrect matches.}

\label{fig:houseplot} 
\end{figure}

In the second part of the experiments, we explore how outliers could
potentially affect the performance of our matching framework. We follow
the protocol in \cite{cho2010}, by randomly select a subset of the
nodes in one of the graphs in matching. Across the sequence, we match
all possible image pairs, space by 10,20,30,40,50,60,70,80,90,100
frames, and compute the average matching accuracy.

Figure \ref{fig:house25} and \ref{fig:house20} show the matching
performance for 25 and 20 randomly selected subset of nodes. It can
be seen that with our optimized kernel and the proposed heat diffusion
distance, our matching performance is greatly improved from the baseline
RRWM algorithm. The matching results in their original work \cite{cho2010}
using RRWM is better than our implementation of RRWM because they
used a different graph structure in matching. The purpose of our comparison
is to show the improvement from using our proposed first order and
second order compatibilities with respect to the baseline adjacency
matrix, while the construction of the adjacency matrices from images
is outside the scope of our paper. The matching accuracy drops when
the number of outliers increases (from 5 to 10), and when the gap increases. 
For smaller gaps with outliers, WKS has a negative effect on matching performance, because WKS has a relatively wide kernel hence more noise tolerant but less informative. In consequence, when the gaps are small, the smoothing WKS kernel instead of enhancing reduces the matching performance. On the contrary, it can be seen that our optimal kernel adapts to the different scenarios much better.

\begin{figure}[ht]
\begin{subfigure}{.48\linewidth} \centering \includegraphics[width=0.96\textwidth]{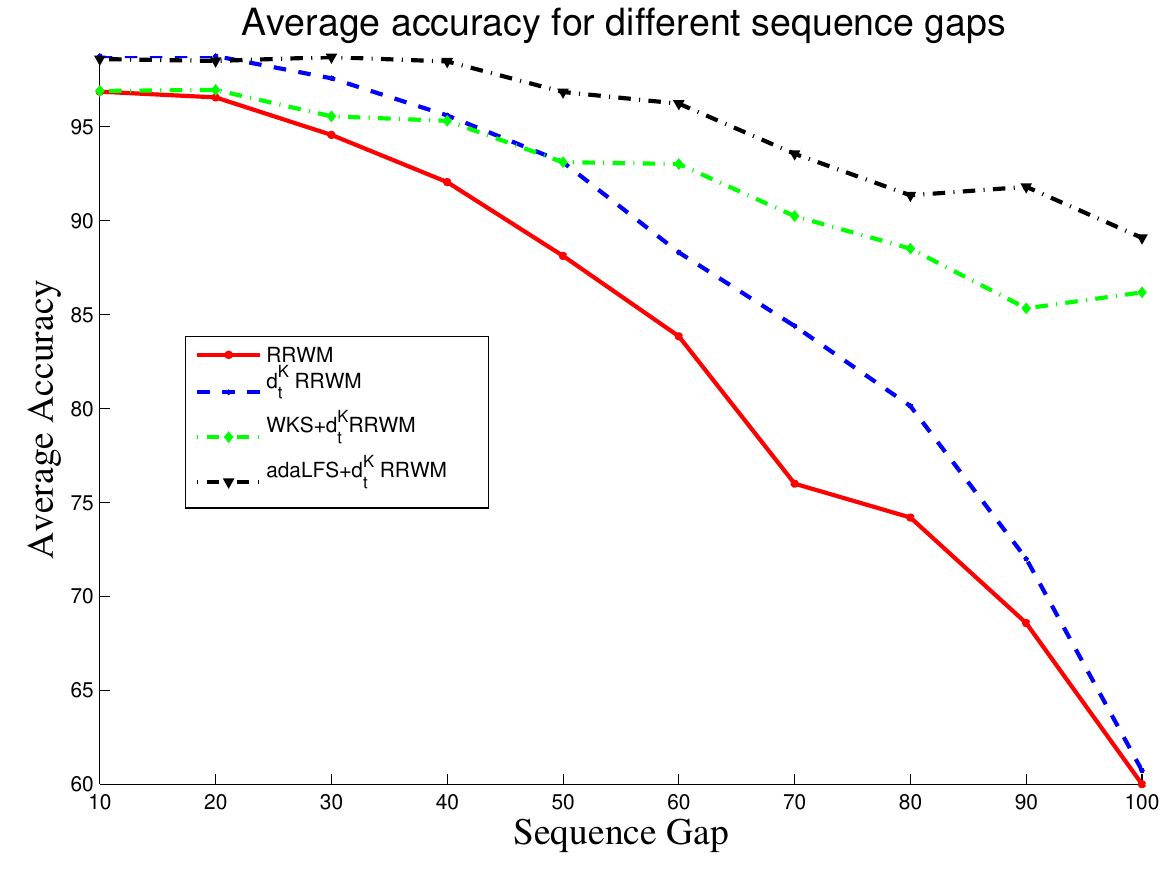}
\caption{30 nodes to 25 nodes}

\label{fig:house25} \end{subfigure} \begin{subfigure}{.48\linewidth}
\centering \includegraphics[width=0.96\textwidth]{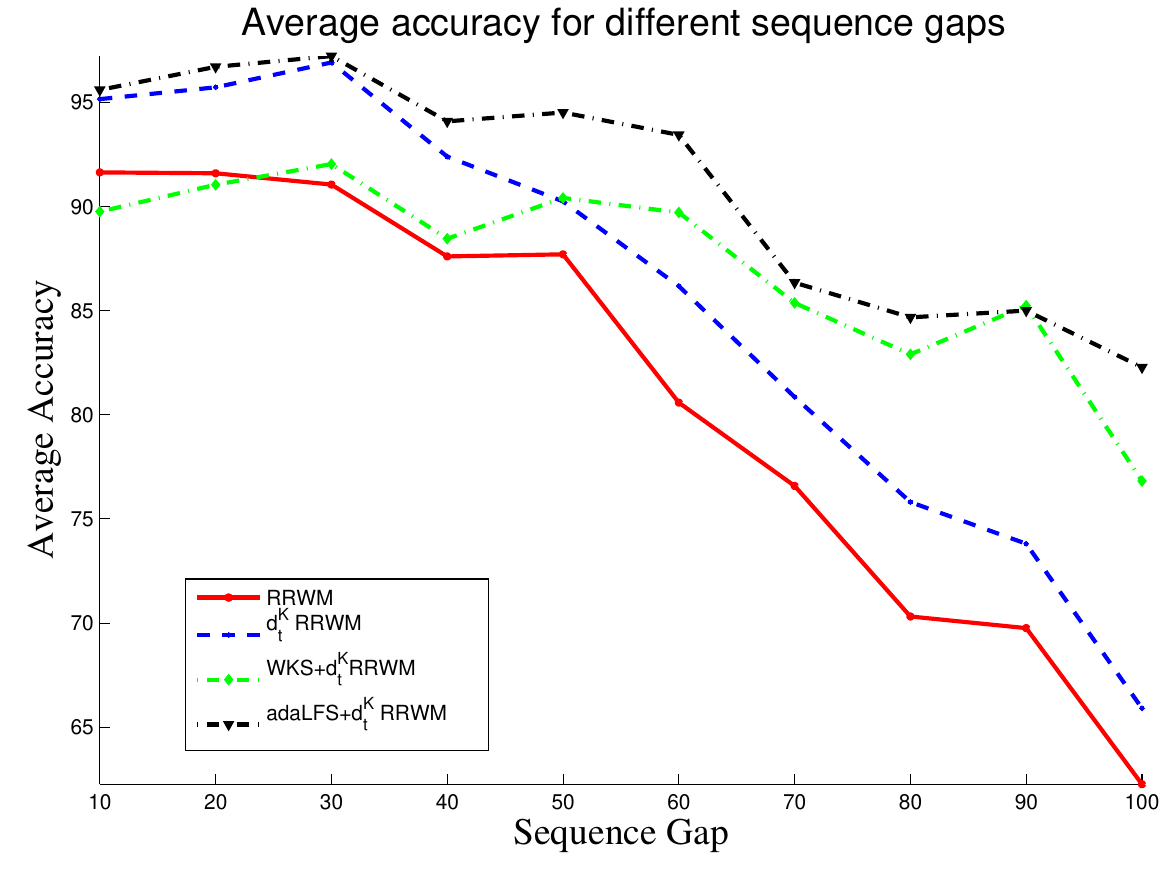}
\caption{30 nodes to 20 nodes}

\label{fig:house20} \end{subfigure}

\caption{Matching with random subset in House sequence.}
\end{figure}

\subsection{Real Image Feature Matching}

In this experiment, we test our descriptor on the real image dataset
used in \cite{cho2010}. This image dataset consists of 30 pairs of
images with labeled feature points. In their original paper, the IQP
affinity matrix $W_{\text{feat}}$ was built considering the similarity
between the appearance-based feature descriptors and the geometric
transformations. Let $d_{\text{feat}}(i,j,a,b)$ be these compatibility
functions from \cite{cho2010}. We add our spectral descriptors as
additional structural compatibility, and set up the affinity matrix
$W$ as 
\[
W_{ia,jb}=\begin{cases}
d_{\text{feat}}(i,j,a,b)+\alpha d_{t}^{\ck}(i,j,a,b) & i\neq j,~a\neq b\\
\beta d(i,a) & i=j,~a=b
\end{cases}
\]
This is done so that we can evaluate the effect of our structural
descriptors on the matching results. In our experiments, we tested
different combination of $\alpha$ and $\beta$. If only considering
$d_{t}^{\ck}$, $\alpha=1$, $\beta=0$ gives the best average accuracy,
and by adding WKS/adaLFS as a node signature constraint, $\alpha=1,\beta=10$
gives the best average accuracy; Table \ref{tab:realmatch} lists
the average accuracies of different methods. Fig. \ref{fig:realimg}
shows an example of matching results.

\begin{table}[ht]
\caption{Average accuracy of real image matching.}

\label{tab:realmatch} \centering %
\begin{tabular}{llll}
\hline 
RRWM  & $d_{t}^{\ck}$  & WKS+$d_{t}^{\ck}$  & adaLFS+$d_{t}^{\ck}$ \tabularnewline
69.92  & 70.68  & 70.70  & 72.57\tabularnewline
\hline 
\end{tabular}
\end{table}

\begin{figure}[ht]
\centering \begin{subfigure}{.32\linewidth} \centering \includegraphics[width=0.96\linewidth]{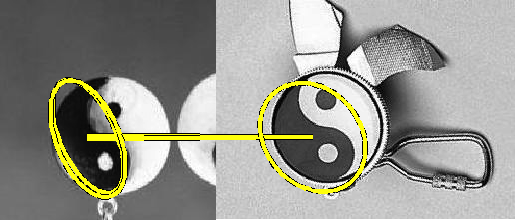}
\caption{RRWM (2/8)}

\end{subfigure} \begin{subfigure}{.32\linewidth} \centering
\includegraphics[width=0.96\linewidth]{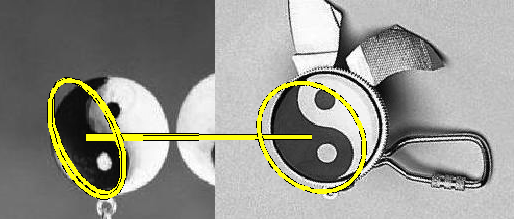} \caption{WKS+$d_{t}^{\ck}$ (2/8)}

\end{subfigure} \begin{subfigure}{.32\linewidth} \centering
\includegraphics[width=0.96\linewidth]{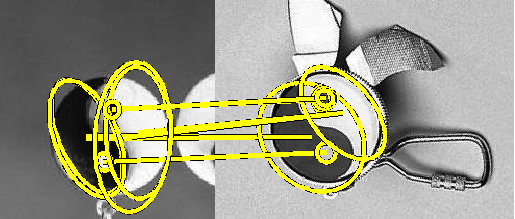} \caption{adaLFS+$d_{t}^{\ck}$ (8/8)}

\end{subfigure} \caption{Matching on real images.}

\label{fig:realimg} 
\end{figure}

\section{Conclusion}

\label{sect:conclude} In this paper, we considered the common quadratic
assignment formulation of weighted graph matching problem, where we
used LFS and the pairwise heat kernel distance as the first and second
order compatibility terms. We have rigorously analyzed their stability
properties; in the case of the first order terms we derived an objective
function that measures both the stability and informativeness of a
given spectral descriptor. By optimizing this objective, we designed
new spectral node signatures tuned to a specific graph to be matched.
Our experiments confirmed that these signatures outperform the existing
spectral node signatures.

This work suggests a number of directions for future research. For
exampe, instead of optimizing the signatures using solely the graph
being matched, it would interesting to explore possibilities for computing
representative graphs for graph collections arising in a given context.
Another direction is to extend our constructions to higher order terms
in the matching scheme \cite{Duchenne2009,Chertok2010}, or to use
them for hypergraph matching \cite{Lee2011}.

{\small \bibliographystyle{ieee}
\bibliography{hks}
 } 
\end{document}